\documentclass[sn-mathphys]{sn-jnl}% Math and Physical Sciences Reference Style
\usepackage{graphicx}%
\usepackage{multirow}%
\usepackage{amsmath,amssymb,amsfonts}%
\usepackage{amsthm}%
\usepackage{mathrsfs}%
\usepackage[title]{appendix}%
\usepackage{xcolor}%
\usepackage{textcomp}%
\usepackage{manyfoot}%
\usepackage{booktabs}%
\usepackage{algorithm}%
\usepackage{algorithmicx}%
\usepackage{algpseudocode}%
\usepackage{listings}%
\usepackage{natbib}%
\usepackage[capitalise,nameinlink]{cleveref}
\usepackage{subfigure}
\usepackage{todonotes}
\usepackage[utf8]{inputenc}

\newcommand{\revision}[1]{{#1}}
\newcommand{\prior}{p_0}

\newcommand{\grad}{\nabla}

\newcommand{\entropy}{\mathcal{H}}
\newcommand{\barloss}{\bar{\ell}}

\newcommand{\loss}{\ell}

\newcommand{\vparam}{\vtheta}

\newcommand{\vnatparam}{\boldsymbol{\lambda}}

\newcommand{\vmeanparam}{\boldsymbol{\mu}}

\newcommand{\dkls}[3]{\mathbb{D}_{\text{KL}}^{#1}[#2 \, \|\, #3]}

\newcommand\cut[1]{}

\newcommand{\elbofinal}{\mathcal{L}}

\newcommand{\tvlambda}{\widetilde{\boldsymbol{\lambda}}}

\newcommand{\squishlist}{
   \begin{list}{$\bullet$}
    { \setlength{\itemsep}{0pt}      \setlength{\parsep}{3pt}
      \setlength{\topsep}{3pt}       \setlength{\partopsep}{0pt}
      \setlength{\leftmargin}{1.5em} \setlength{\labelwidth}{1em}
      \setlength{\labelsep}{0.5em} } }

\newcommand{\squishlisttwo}{
   \begin{list}{$\bullet$}
    { \setlength{\itemsep}{0pt}    \setlength{\parsep}{0pt}
      \setlength{\topsep}{0pt}     \setlength{\partopsep}{0pt}
      \setlength{\leftmargin}{2em} \setlength{\labelwidth}{1.5em}
      \setlength{\labelsep}{0.5em} } }

\newcommand{\squishend}{
    \end{list}  }

{}
\newtheorem{thm}{Theorem}{}
{}
{}

\newcommand{\half}{\mbox{$\frac{1}{2}$}}

\newcommand{\real}{\mbox{$\mathbb{R}$}}

\newcommand{\rnd}[1]{\left(#1\right)}
\newcommand{\sqr}[1]{\left[#1\right]}

\newcommand{\myang}[1]{\langle#1\rangle}
\newcommand{\myexpect}{\mathbb{E}}

\newcommand{\gauss}{\mbox{${\cal N}$}}

\newcommand{\myvec}[1]{\mbox{$\mathbf{#1}$}}
\newcommand{\myvecsym}[1]{\mbox{$\boldsymbol{#1}$}}

\newcommand{\vzero}{\mbox{$\myvecsym{0}$}}
\newcommand{\vone}{\mbox{$\myvecsym{1}$}}

\newcommand{\vepsilon}{\mbox{$\myvecsym{\epsilon}$}}

\newcommand{\vmu}{\mbox{$\myvecsym{\mu}$}}

\newcommand{\vlambda}{\mbox{$\myvecsym{\lambda}$}}

\newcommand{\vtheta}{\mbox{$\myvecsym{\theta}$}}

\newcommand{\vm}{\mbox{$\myvec{m}$}}

\newcommand{\vs}{\mbox{$\myvec{s}$}}

\newcommand{\vv}{\mbox{$\myvec{v}$}}

\newcommand{\vy}{\mbox{$\myvec{y}$}}

\newcommand{\vF}{\mbox{$\myvec{F}$}}

\newcommand{\vH}{\mbox{$\myvec{H}$}}
\newcommand{\vI}{\mbox{$\myvec{I}$}}

\newcommand{\vS}{\mbox{$\myvec{S}$}}
\newcommand{\vT}{\mbox{$\myvec{T}$}}

\newcommand{\vX}{\mbox{$\myvec{X}$}}
\newcommand{\diag}{\mbox{$\mbox{diag}$}}

\newcommand{\trace}{\mbox{Tr}}

\newcommand{\calD}{\mbox{${\cal D}$}}

\newcommand{\data}{\calD}

\newcommand{\be}{\begin{equation}}
\newcommand{\ee}{\end{equation}}
\newcommand{\bea}{\begin{eqnarray}}
\newcommand{\eea}{\end{eqnarray}}
\newcommand{\beaa}{\begin{eqnarray*}}
\newcommand{\eeaa}{\end{eqnarray*}}

\DeclareMathOperator*{\argmin}{arg\,min}

\crefname{section}{Sec.}{Sections}
\crefname{appendix}{App.}{Appendices}
\crefname{algorithm}{Alg.}{Algorithms}
\crefname{equation}{Eq.}{Eqs.}
\crefname{figure}{Fig.}{Figures}
\creflabelformat{equation}{#2\textup{#1}#3} % Remove the brackets around Equation numbers

\usepackage{thmtools}
\declaretheorem[
mdframed={
  skipabove=6pt,
  skipbelow=6pt,
  hidealllines=true,
  backgroundcolor={lightgray},
  innerleftmargin=8pt,
  innerrightmargin=8pt}
]{ex}

\theoremstyle{thmstyleone}%
\theoremstyle{thmstyletwo}%

\theoremstyle{thmstylethree}%

\begin{document}

\title[Information Geometry of Variational Bayes]{Information Geometry of Variational Bayes}


\author{\fnm{Mohammad Emtiyaz} \sur{Khan}}\email{emtiyaz.khan@riken.jp}

\affil{\orgdiv{RIKEN Center for AI Project}, \orgaddress{\street{1-4-1 Nihonbashi, Chuo-ku}, \city{Tokyo}, \postcode{174-0055}, \state{Tokyo}, \country{Japan}}}

%\affil[2]{\orgdiv{Department}, \orgname{Organization}, \orgaddress{\street{Street}, \city{City}, \postcode{10587}, \state{State}, \country{Country}}}
%
%\affil[3]{\orgdiv{Department}, \orgname{Organization}, \orgaddress{\street{Street}, \city{City}, \postcode{610101}, \state{State}, \country{Country}}}


\abstract{We highlight a fundamental connection between information geometry and variational Bayes (VB) and discuss its consequences for machine learning. Under \revision{certain} conditions, a VB solution always requires estimation or computation of natural gradients. We show several consequences of this fact by using the natural-gradient descent algorithm of \citet{khan2023bayesian} called the Bayesian Learning Rule (BLR). These include (i) \revision{a} simplification of Bayes' rule as addition
of natural gradients,
(ii) \revision{a generalization} of quadratic surrogates used in gradient-based methods, and (iii) \revision{a large-scale implementation} of VB algorithms \revision{for} large language models. Neither the connection nor \revision{its} consequences are new but \revision{we further emphasize the common origins of the two fields of information geometry and Bayes} with a hope to facilitate more work at the intersection of the two fields.}

\keywords{Bayesian Inference, Variational Inference, Natural Gradient Descent}



\maketitle

\section{Introduction}\label{sec1}

The information geometry of variational Bayes (VB) \citep{saul1996mean, bishop2006pattern} is a relatively less studied area compared to maximum-likelihood estimation \citep{amari1998natural}. Some of the early \revision{works on} natural gradients for VB were proposed as early as \revision{the} 90s \citep{amari1995information, sato1999, sato2001online} but they \revision{have} only \revision{been} popularized recently, for instance, by \citet{honkela2010approximate} and \citet{hoffman2013stochastic}, among others. 
These works clearly show that the two fields of VB and information geometry are inseparable.
Our goal in this paper is to highlight this fundamental connection and show its impact on machine learning. Both the connection and its consequences are based on existing works. \revision{The goal of this work is to further emphasize the common origins of the two fields of information geometry and Bayes.} \revision{Our} hope is that the connection will facilitate more work at the intersection of the two fields.

We start in \cref{sec:background} with some background on VB and then, in \cref{sec:igvb}, illustrate the special role of information geometry and natural gradients for VB. In \cref{sec:blr}, we present the Bayesian Learning Rule (BLR) \citep{khan2023bayesian} which exploits the power of natural gradients to optimize the VB objective. \revision{Subsequently}, we discuss three consequences due to the connection\revision{s} between information geometry
and VB.  
\begin{enumerate}
   %(\cref{eq:blr}).
   \item In \cref{sec:add}, we show that Bayes' rule for conjugate models can be equivalently written as addition of natural gradients.
   \item In \cref{sec:beyondconj}, we show that \revision{the} natural gradients used in the BLR generalize the quadratic surrogates used in optimization methods.
   \item In \cref{sec:scaling}, we highlight the impact of \revision{the above results} on scaling VB to large deep networks (such as GPT-2) which was not possible until recently \citep{ShenDCNMBYGCKM24}. 
\end{enumerate}
Throughout, we emphasize the role played by information geometry in VB. Our hope is that these results \revision{will} facilitate future research at the intersection of the two fields of information geometry and Bayesian learning. 

%The relation between the Bayes and information geometry can be understood with a simple fact: multiplication in the distribution space can (often) be realized by addition in the natural-gradient space. This fact can then be used to not only connect natural-gradient descent (NGD) to Bayes' rule but also to show that natural-gradients naturally arise in various seemingly unrelated  updating schemes \citep{khan2023bayesian}. The connection also reveal some unexplored
%directions in information geometry which could lead to important new research problems.

\section{Variational Bayesian Learning}
\label{sec:background}

\revision{Variational Bayesian (VB) inference aims to} approximate the posterior distribution by using an optimization problem over the space of distributions. For example, consider a Bayesian model with likelihood $p(\data|\vparam)$ over the data $\data$ with parameter vector $\vparam\in\real^P$ and a prior $\prior(\vparam)$. For such models, \citet{zellner1988optimal} \revision{shows} that the posterior distribution $p(\vparam|\data)$ obtained by using Bayes' rule can also be recovered by solving an
optimization problem over the space $\mathcal{P}$ of all possible distributions over $\vparam$,
\begin{equation}
   p(\vparam|\data) = \frac{p(\data|\vparam) \prior(\vparam)}{p(\data)} = \argmin_{q \in \mathcal{P}}  \,\,
   \myexpect_{q} \sqr{ -\log p(\data|\vparam)}
    + \dkls{}{q}{\prior}.
    \label{eq:BayesP}
\end{equation}
The optimization in \revision{the last expression} is over distributions $q(\vparam)\in \mathcal{P}$, and the second term \revision{there} denotes the Kullback-Leibler (KL) divergence between $q(\vparam)$ and $\prior(\vparam)$.

VB restricts the optimization problem \revision{to approximate the posterior within} a subset $\mathcal{Q} \subset \mathcal{P}$ of candidate distributions, for example, $\mathcal{Q}$ can be the set of all Gaussian distributions. \revision{Below, we will write the VB formulation in a general form using a loss function,} which is suited for machine learning and deep learning problems. Essentially, we rearrange the objective by expanding the KL term and
\revision{clubbing} the likelihood and prior together as shown below:
\[
   \myexpect_{q} \sqr{ -\log p(\data|\vparam)} + \dkls{}{q}{\prior} = \myexpect_q\sqr{-\log \rnd{p(\data|\vparam)\prior(\vparam)} } - \entropy(q).
\]
The second term on the right is the entropy denoted by $\entropy(q)$. Then, by defining \revision{the loss as} $\barloss(\vparam) = - \log \rnd{ p(\data|\vparam) \prior(\vparam)}$, the optimization problem in \cref{eq:BayesP} can be rewritten as the following VB problem over the \revision{subset} of distributions $\mathcal{Q}$, 
\begin{equation}
   q_*(\vparam) = \argmin_{q \in \mathcal{Q}} \,\, \myexpect_{q} \sqr{  \barloss(\vparam) } - \entropy(q).
   \label{eq:vb}
\end{equation}
We refer to this as the VB learning problem (as opposed to VB inference).
The form applies more generally than VB inference, extending to cases when $\barloss$ does not correspond to a `proper' Bayesian model. Such cases are often covered under a larger Bayesian framework that target\revision{s} a \emph{generalized} posterior which is proportional to $\exp(-\barloss(\vparam))$ with an arbitrary loss \citep{10.1145/307400.307433, catoni2007pac, art670}. 

%For example, we can replace the joint $p(\data|\vparam)p(\vparam)$ by a generic
%Gibbs distribution \todo{Normalized or not normalized?}
%$e^{-\barloss(\text{\vparam})}$ with a loss $\barloss(\vparam)$.
%\todo{Explain $\barloss$, where did $D$ go?}
%To handle such cases, we can rewrite 
%\[
%   \myexpect_{q} \sqr{ -\log p(\data|\vparam)} + \dkls{}{q(\vparam)}{p(\vparam)} = \myexpect_q\sqr{-\log \rnd{p(\data|\vparam)p(\vparam)} } - \entropy(q)
%\]
%where $\entropy(q)$ is the entropy of $q$. Then, we can replace the first term by $\myexpect_q[\barloss(\vparam)]$. This then yields the version of \cref{eq:BayesP} which optimizes over the set $\mathcal{Q}$, 
%\begin{equation}
%   q_*(\vparam) = \argmin_{q \in \mathcal{Q}} \,\, \myexpect_{q} \sqr{  \barloss(\vparam) } - \entropy(q).
%   \label{eq:vb}
%\end{equation}
%We refer to the above as the \emph{VB objective} and use it to find best approximations $q_*(\vparam)$ to the posterior $p(\vparam|\data)$ (or the Gibbs distribution $e^{-\barloss(\text{\vparam})}$ in general). 

\section{Information Geometry of VB Solutions}
\label{sec:igvb}

The goal of this paper is to highlight a fundamental connection between variational Bayes and information geometry. We will now show that computation of any $q_*$ must involve computation of \emph{natural} gradients, which are gradients scaled by a Fisher matrix. We believe \revision{that} this result \revision{holds} for general posterior forms, but here we will focus on cases where $\mathcal{Q}$ is restricted to regular and minimal exponential-family
\revision{members} with respect to a fixed sufficient statistic $\vT(\vparam)$ as shown below:
\begin{equation}
   q_{\vnatparam}(\vparam) = h(\vparam) \exp\sqr{\myang{ \vnatparam,
            \vT(\vparam) } - A(\vnatparam)}.
    \label{eq:exp_fam}
\end{equation}
The candidates are parameterized through $\vnatparam \in \Omega$ which is called the natural parameter and \revision{takes values} in a non-empty open set $\Omega$. The cumulant function $A(\vnatparam)$ is finite, strictly convex and differentiable over $\Omega$. Further, $\myang{\cdot,\cdot}$ is an inner product and $h(\vparam)$ is a function of $\vparam$. The family is log-linear, that is, \revision{the} log of $q_{\vnatparam}(\vparam)$ is linear in $\vT(\vparam)$.

Natural gradients naturally arise for VB with exponential-family posterior distributions. This is because the derivative of the Shannon entropy $\entropy(q_{\vnatparam})$ with respect to natural parameters $\vnatparam$ necessarily involves the Fisher matrix of $q_{\vnatparam}$. This is derived below (by assuming $h(\vparam) = 1$ for simplicity),
\begin{equation}
   \begin{split}
      \grad_{\vnatparam} \entropy(q_{\vnatparam}) 
      &= -\grad_{\vnatparam} \myexpect_{q_{\vnatparam}}[ \log q_{\vnatparam}(\vparam) ] \\
      &= - \grad_{\vnatparam} \myexpect_{q_{\vnatparam}}[ \myang{ \vnatparam, \vT(\vparam) } -  A(\vnatparam) ] \\
      &= - \nabla_{\vnatparam} \sqr{ \myang{ \vnatparam, \nabla_{\vnatparam} A(\vnatparam) } - A(\vnatparam) }\\
      &= - [\nabla_{\vnatparam}^2 A(\vlambda)]\, \vlambda - \nabla_{\vnatparam} A(\vnatparam)  + \nabla_{\vnatparam} A(\vnatparam) \\
      %= \nabla A(\vnatparam) - \vmu - \sqr{ \nabla_{\vnatparam} \vmu } \vlambda  
      &= - [\nabla_{\vnatparam}^2 A(\vlambda)]\, \vlambda.
   \end{split}
   \label{eq:grad_ent}
\end{equation}
The second equality is obtained by simply substituting the definition of $q_{\vnatparam}$. The third equality is due to a well known property of exponential-families that the \revision{expected value of sufficient statistic} is simply the derivative of the cumulant function, \revision{that is, $\myexpect_{q_{\vnatparam}}[\vT(\vparam)] = \nabla_{\vnatparam} A(\vnatparam)$}. \revision{The expectation $\myexpect_{q_{\vnatparam}}[\vT(\vparam)]$ is
also known as the expectation parameter, as we will
see in the next section where we will
denote it by $\vmu(\vlambda)$. The fourth line applies} the derivative, and the last line is obtained by cancelling the last two terms. For
exponential-families, the Fisher matrix is
$\vF(\vlambda) = \nabla_{\vnatparam}^2 A(\vnatparam)$, \revision{which shows} that the gradient of the entropy always contains this matrix. 

Using this result, we can obtain the optimality condition of VB by simply setting the derivative of the objective at the optimal $q_*$ to zero. Denoting the natural parameter of $q_*$ by $\vlambda_*$ and rewriting $q_*$ as $q_{\vnatparam_*}$, the optimality condition becomes $\nabla_{\vnatparam_*} \myexpect_{q_{\vnatparam_*}}[\barloss(\vparam)] = \nabla_{\vnatparam_*} \entropy(q_{\vnatparam_*})$. Then we \revision{substitute} \cref{eq:grad_ent} \revision{in} the right hand side to get
\begin{equation}
   \vnatparam_* = \vF(\vnatparam_*)^{-1} \grad_{\vnatparam_*} \myexpect_{q_*} [-\barloss(\vparam)]. 
   \label{eq:vb_opt}
\end{equation}
\revision{The quantity $\vF(\vlambda)^{-1} \nabla_{\vnatparam}$ on the right is a natural gradient due to the preconditioning by the inverse Fisher.} The equation shows that the optimal natural parameter is equal to the natural gradient of \revision{the} expected negative-loss at $\vlambda_*$. This implies that the computation of $\vlambda_*$ necessarily involves \revision{the} inversion of the Fisher matrix.

%The same result applies to the computation of $q_*$ too. This can be made clear by plugging in the value of $\vlambda_*$ in exponential form of \cref{eq:exp_fam} to write,
%\begin{equation}
%   q_{\vnatparam_*}(\vparam) \propto \exp\sqr{\myang{ \vnatparam_*, \vT(\vparam) }} = \exp\sqr{ \vF(\vnatparam_*)^{-1} \grad_{\vnatparam_*} \myexpect_{q_*} [-\barloss(\vparam)], \vT(\vparam)}. 
%   \label{eq:q_opt}
%\end{equation}

This result was first derived \revision{by} \citet{salimans2013fixed} \revision{who draw} an analogy to the solution of linear regression. For example, consider a linear model $\vy = \vX\vparam + \vepsilon$ with output vector $\vy$, \revision{feature} matrix $\vX$, and \revision{error vector} $\vepsilon$. Then, the optimal parameter is given by $\vparam_* = (\vX^\top\vX)^{-1} \vX^\top \vy$ which, similarly to \cref{eq:vb_opt}, involves an inversion.
\revision{Regardless} of the choice of algorithm, computing $\vparam_*$ necessarily requires us to solve a linear system. The convergence rate of an algorithm will therefore depend on the conditioning of the Gram matrix. A Newton step will immediately converge in one iteration, while gradient descent will take multiple steps, but both will solve the same linear system of equations. There is no escaping it. 

By analogy, any algorithm estimating $\vlambda_*$ through VB is forced to deal with \revision{the geometry of $q_{\vnatparam_*}$.} That is, the algorithm must solve \cref{eq:vb_opt} and \revision{is forced to compute or estimate the required} natural gradient. In this sense, any gradient based VB method \citep{graves2011practical,blundell2015weight} or even the black-box methods \citep{ranganath2013black} all aim to compute the same natural gradient. The conditioning of the Fisher
matrix \revision{thus} contributes to their convergence behavior.
Methods that ignore the geometry are expected to be slower
than those that exploit it to their advantage. 
%The quantity on the right hand side of \cref{eq:vb_opt} is referred to as natural gradient \citep{amari1998natural} which we denote by
%\begin{equation}
%   \natgrad_{\vnatparam} = \vF(\vnatparam)^{-1} \grad_{\vnatparam}
%\end{equation}
%The optimality condition in \cref{eq:vb_opt} therefore shows that optimal natural parameter is equal to the natural gradient of the expected loss (evaluated at the optimal natural parameter).
In summary, natural gradients are an integral part of all VB solutions and all VB methods must deal with computation or estimation of natural gradients, either directly or indirectly.

\section{Information Geometry of the BLR}
\label{sec:blr}

\revision{The result in \cref{eq:vb_opt} highlights the information geometry of VB solutions but a similar result also holds for the \emph{iterates} obtained during the optimization of the VB objective. We will now derive the latter result by using a natural-gradient descent algorithm. Specifically, we will use the Bayesian Learning Rule (BLR) of \citet{khan2023bayesian} which takes a form similar to the Bayes' rule.}

To write the natural gradient descent in a computationally convenient form, we will exploit the `dual' coordinates of exponential families. \revision{This} form avoids an explicit inversion of the Fisher matrix\revision{, and can also be implemented easily. The dual coordinates are the expectation parameters defined as} $\vmeanparam(\vnatparam) = \myexpect_q[\vT(\vparam)] = \nabla A(\vnatparam)$ \revision{which we encountered while deriving \cref{eq:grad_ent}.} 
The mapping between the pair $(\vnatparam, \vmeanparam)$ is a bijection, \revision{which makes $\vmu$ a different} way to parameterize $q_{\vnatparam}$ \citep{banerjee2005clustering}.

The advantage of the dual coordinates is that they enable us to \revision{express} natural gradients with respect to $\vlambda$ as vanilla gradients with respect to $\vmu$, as shown below,
\begin{equation}
   \vF(\vnatparam)^{-1} \grad_{\vnatparam} = {\vF(\vnatparam)}^{-1}\sqr{ \grad_{\vnatparam} \vmeanparam(\vlambda)} \grad_{\vmeanparam} = \sqr{\grad_{\vnatparam}^2 A(\vnatparam) }^{-1}\sqr{ \grad_{\vnatparam}^2 A(\vlambda)} \grad_{\vmeanparam} = \nabla_{\vmeanparam} .
   \label{eq:equivalence_natgrad}
\end{equation}
The first equality follows simply from the chain rule to switch to the gradient with respect to $\vmu$, while the second equality is obtained by using the result that \revision{$\vF(\vlambda) = \grad_{\vnatparam} \vmu(\vlambda) = \nabla_{\vnatparam}^2 A(\vlambda)$}. The result shows that the natural gradient at $\vlambda$ is equivalent to vanilla gradient with respect to $\vmu$ at $\vmu(\vlambda) = \nabla_{\vnatparam} A(\vlambda)$.
\revision{Thus}, we can rewrite the optimality condition in \cref{eq:vb_opt} as follows:
\begin{equation}
   \vnatparam_* = \grad_{\vmeanparam_*} \myexpect_{q_{\vnatparam_*}} [-\barloss(\vparam)],  
   \label{eq:vb_opt_mu}
\end{equation}
where we denote $\vmu_* = \nabla_{\vnatparam} A(\vlambda_*)$. This provides yet another way to understand the information geometry of VB solutions: the natural parameter is simply equal to the gradients computed in the dual $\vmu$-space.  

The Bayesian Learning Rule uses the following iterations \revision{with a learning rate $\rho_t>0$ at each iteration $t$ to find the solution},
\begin{equation}
   \vnatparam_{t+1} \leftarrow (1-\rho_t)\vnatparam_t + \rho_t \grad_{\vmeanparam_t} \myexpect_{q_{\vnatparam_t}} [-\barloss(\vparam)] ,
    \label{eq:blr}
\end{equation}
where $\vmu_t = \nabla A(\vlambda_t)$.
It is easy to check that, when these steps converge, the optimal $\vlambda_*$ satisfies \cref{eq:vb_opt_mu} (assuming that $\vnatparam_{t}$ are `valid' natural parameters for all $t$). The algorithm performs natural-gradient descent in the $\vnatparam$-space due to the property shown in \cref{eq:equivalence_natgrad} and it is also a mirror-descent algorithm in the $\vmu$-space as originally proposed by \cite{khan2017conjugate} in the context of variational inference. 
% The BLR extends this application to a generic loss function, presenting it as a generic learning rule for machine learning.

\section{Bayes' Rule as Addition of Natural Gradients}
\label{sec:add}

The BLR uses a convenient `additive' update where natural-gradients are added to the \revision{previous} natural parameters, but the same update can be written in a multiplicative form that resembles Bayes' rule. This reformulation enables us to then connect additive updates used in machine learning to multiplicative updates used in Bayes' rule. 

To rewrite the BLR in a multiplicative form, we simply need to \revision{substitute} $\vlambda_{t+1}$ from \cref{eq:blr} into \cref{eq:exp_fam} and simplify to get
\begin{equation}
    q_{t+1}(\vparam) \propto q_t(\vparam)^{(1- \rho_t)}
    \sqr{ \exp\rnd{ { \myang{ \tvlambda_t \, , \, \text{\vT}(\text{\vparam})}}} }^{\rho_t}
    \text{ where } \tvlambda_t = \grad_{\vmeanparam_t} \myexpect_{q_t} [-\barloss(\vparam)].
    \label{eq:blr_dist}
\end{equation}
\revision{Here, to simplify the notation, we denote $q_{\vnatparam_t}$ by $q_t$.}
\revision{This} form is similar to a Bayes' filter where the previous $q_t$ is used as the prior and multiplied \revision{with} a likelihood \revision{that has an exponential-family} form \revision{using a} natural parameter $\tvlambda_t$. Essentially, every step of the BLR converts the loss $\barloss$ into a likelihood with an exponential\revision{-family} form. The step can therefore be seen as Bayesian inference on a `conjugate' exponential-family model. This interpretation was originally proposed by \citet{khan2017conjugate} and we will use it to argue that natural gradients are
also \revision{inherently present in} Bayes' rule.

We first illustrate the \revision{main} idea on a simple ridge-regression example where the computation of the posterior can be written as addition of natural parameters. Afterward, we will show that this is equivalent to addition of natural gradients and provide a formal statement of the result.
%
%\citet{khan2023bayesian} show that many machine-learning algorithms can be seen as special instances of the update in \cref{eq:blr}. As a consequence, all such algorithms are therefore seen as optimizing the VB objective in \cref{eq:vb} by computing some approximation to the natural gradients used in \cref{eq:blr}.
%Our goal here is to highlight the role of information geometry in achieving this result. Essentially, the NGD steps in \cref{eq:blr} are an extension of the
%Bayes' rule in the \emph{log}-space: addition of natural-gradients plays a similar role to multiplication in the Bayes' rule. Due to this reason, algorithms using an addition in the $\vnatparam$-space can be seen as optimizing the VB objective. We will now give a simple example to demonstrate this, and then discuss consequences for deep learning in the next section.
%
\begin{ex}
   \label{ex:ridge}
   Consider ridge regression for $\data = (\vy,
   \vX)$ with output vector $\vy\in\real^N$ and feature matrix $\vX\in\real^{N\times P}$. The likelihood and prior are given as follows:
   \[
      p(\vy|\vparam) = \gauss(\vy|\vX\vparam, \vI_N)
      \qquad
      \prior(\vparam) = \gauss(\vparam|\vzero, \vI_P).
   \]
   where $\vI_N$ and $\vI_P$ are identity matrices of size $N$ and $P$ respectively.
   We can write both densities in a quadratic form with respect to $\vparam$ as shown below,
   \begin{equation}
      \begin{split}
         p(\vy|\vparam) &\propto \exp \sqr{ \vy^\top\vX \vparam + \trace\rnd{( -\half \vX^\top\vX) \vparam\vparam^\top}} \\
         \prior(\vparam) &\propto \exp \sqr{ \vzero^\top \vparam + \trace\rnd{( -\half\vI_P) \vparam\vparam^\top}}. \\
      \end{split}
   \end{equation}
   Multiplying the two densities, we can write their product in the same form too,
   \begin{equation}
      p(\vparam|\vy) \propto p(\vy|\vparam) \prior(\vparam) 
      \propto \exp \sqr{ \rnd{ \vX^\top \vy + \vzero}^\top \vparam + \trace\rnd{- \half( \vX^\top\vX + \vI_P) \vparam\vparam^\top}}. 
      \label{eq:post1}
   \end{equation}
   Because the product is quadratic in $\vparam$, we can conclude that the posterior also takes a Gaussian form. This property is often referred to as `conjugacy' in the Bayesian literature, and we can equivalently write any such conjugate operation as addition of natural parameters.

   To show this, let us denote the Gaussian posterior by $\gauss(\vparam| \vm_*, \vS_*^{-1})$ with mean $\vm_*$ and precision $\vS_*$, which can also be written in an exponential-family form,
   \begin{equation}
      \gauss(\vparam|\vm_*, \vS_*^{-1}) \propto \exp \sqr{ \vm_*^\top\vS_* \vparam + \trace\rnd{ - \half\vS_*\vparam\vparam^\top}}.
      \label{eq:post2}
   \end{equation}
   Now, by simply equating the terms in front of $\vparam$ and $\vparam\vparam^\top$ in \cref{eq:post1} and \cref{eq:post2} respectively, we can write the multiplication as an addition in the $\vlambda$-space,
   \begin{equation}
      \underbrace{ \rnd{ \begin{array}{c} \vS_*\vm_*\\ -\half \vS_* \end{array} } }_{\vnatparam_*}
         = \underbrace{ \rnd{ \begin{array}{c} \vX^\top \vy \\ -\half \vX^\top \vX \end{array} } }_{\tvlambda_{\text{lik}}}
            + \underbrace{ \rnd{ \begin{array}{c} \vzero\\ -\half \vI_P \end{array} } }_{\tvlambda_{\text{prior}}} 
               \quad\implies \quad
               \vnatparam_* = \tvlambda_{\text{lik}} + \tvlambda_{\text{prior}}
   \end{equation}
   The natural parameter $\vnatparam_*$ \revision{is thus} written as a sum of $\tvlambda_{\text{lik}}$ and $\tvlambda_{\text{prior}}$.
\end{ex}
\vspace{.3cm}
In general, for conjugate models, the likelihood, prior, and posterior all take the same exponential-\revision{family} form with respect to some sufficient statistic $\vT(\vparam)$. That is, for some \smash{$\tvlambda_{\text{lik}}, \tvlambda_{\text{prior}}$}, and $\vnatparam_*$, we can write them as follows:
\begin{equation}
   \begin{split}
      p(\vy|\vparam) &\propto \exp\sqr{\myang{ \tvlambda_{\text{lik}},
               \text{\vT}(\text{\vparam}) }}, \\
      p(\vparam) &\propto \exp\sqr{\myang{ \tvlambda_{\text{prior}},
               \text{\vT}(\text{\vparam}) }}, \\
      p(\vparam|\vy) &\propto \exp\sqr{\myang{ \vnatparam_*,
               \text{\vT}(\text{\vparam}) }}, 
   %            \label{eq:exp-lik-prior-post}
   \end{split}
\end{equation}
and Bayes' rule can be expressed as a simple addition 
\begin{equation}
   p(\vparam|\vy) \propto p(\vy|\vparam) \prior(\vparam)
   \quad \iff \quad
   \vnatparam_* = \tvlambda_{\text{lik}} + \tvlambda_{\text{prior}}.
   \label{eq:bayes_rule_add}
\end{equation}
%\citet{khan2017conjugate} referred to this as \emph{conjugate computation} and showed that the BLR update in \cref{eq:blr} generalize them to cases where conjugacy does not hold, for example, in deep learning. 
%
The quantities being added above can also be seen as natural gradients. \revision{Below, we show that} the natural gradients of the expected log-likelihood and expected log-prior are $\tvlambda_{\text{lik}}$ and $\tvlambda_{\text{prior}}$ respectively:
\begin{equation}
   \begin{split}
      \grad_{\vmeanparam} \myexpect_{q_{\vnatparam}} [\log p(\vparam|\vy)]  
      &= \grad_{\vmeanparam} \myexpect_{q_{\vnatparam}} [{\myang{ \tvlambda_{\text{lik}}, \text{\vT}(\text{\vparam}) }}] 
      = \grad_{\vmeanparam} [{\myang{ \tvlambda_{\text{lik}}, \vmu }}] 
      = \tvlambda_{\text{lik}} \\
      \grad_{\vmeanparam} \myexpect_{q_{\vnatparam}} [\log \prior(\vparam)]  
      &= \grad_{\vmeanparam} \myexpect_{q_{\vnatparam}} [{\myang{ \tvlambda_{\text{prior}}, \text{\vT}(\text{\vparam}) }}] 
      = \grad_{\vmeanparam} [{\myang{ \tvlambda_{\text{prior}}, \vmu }}] 
      = \tvlambda_{\text{prior}} 
   \end{split}
\end{equation}
In general, whenever $\barloss(\vparam)$ is linear in $\vT(\vparam)$, then \revision{the} natural gradient \revision{of $\myexpect_{q_{\vnatparam}}[\loss(\vparam)]$} is equal to the factor that is in front of $\vT(\vparam)$. The following theorem formalizes this.
\vspace{.3cm}
\begin{thm}
   For any loss of \revision{the} form $\barloss(\text{\vparam}) \propto \myang{ -\tvlambda,
            \text{\vT}(\text{\vparam}) } $, we have $\grad_{\vmeanparam} \myexpect_{q_{\vnatparam}} [\loss(\vparam)] = -\tvlambda$.
\end{thm}
\vspace{.3cm}
Due to this property, Bayes' rule can be \revision{seen} as a special case of the BLR with learning rate $\rho_t = 1$. The following theorem further shows that \revision{the BLR converges in one step, regardless of where it is initialized.}
\vspace{.3cm}
\begin{thm} 
   Bayes' rule in conjugate models is realized by one step of the BLR update in \cref{eq:blr} with learning rate $\rho_t = 1$.
\end{thm}
\vspace{.3cm}
\revision{The one-step} convergence is a direct consequence of the fact that the posterior is available in closed form. 
In general, the BLR in \cref{eq:blr_dist} extends the closed-form multiplicative property of Bayes' rule to \revision{general cases}. This is made possible by exploiting the information geometry of VB \revision{through} natural gradients. At each iteration, natural gradients are used to write an update that resembles Bayes' rule which is also equivalent to a simple addition:
\begin{equation}
   q_{t+1}(\vparam) \propto q_t(\vparam)^{(1- \rho_t)}
    \sqr{ \exp\rnd{ { \myang{ \tvlambda_t \, , \, \text{\vT}(\text{\vparam})}}} }^{\rho_t}
   \quad \iff \quad
   \vnatparam_{t+1} = (1-\rho_t) \vnatparam_t + \rho_t \tvlambda_t .
\end{equation}
This is an extension of \cref{eq:bayes_rule_add} to \revision{problems involving non-conjguate terms $\barloss$ and where no closed-form expression for the posterior exists.} 
The optimization \revision{will} not terminate in one step, but the form still resembles Bayes' rule.

\section{Beyond Quadratic Surrogates with BLR}
\label{sec:beyondconj}

The \revision{additive} property of the BLR enables us to derive machine learning algorithms as special cases of natural gradient descent on VB, as shown by \citet{khan2023bayesian}. We show now that this is made possible because the BLR yields a surrogate \revision{which is similar to} the quadratic surrogates used in standard continuous optimization methods. Information geometry plays an important role here \revision{because the
BLR surrogates are obtained by replacing vanilla gradients 
of the loss by natural gradients.}

We start by discussing the standard quadratic-surrogate model used in popular optimization algorithms. For instance, the Newton step can be written as minimization of a local quadratic surrogate defined at $\vparam_t$,  
\begin{align}
   %\vparam_{t+1} &= \vparam_t - \alpha_t \nabla \loss(\vparam_t) &&= \arg\min_{\vparam} \,\, \vparam^\top \grad \barloss(\vparam_t) + \frac{1}{2\alpha_t} (\vparam-\vparam_t)^\top (\vparam-\vparam_t) \label{eq:grad_desc}\\
   \vparam_{t+1} &= \vparam_t - \vH(\vparam_t)^{-1} \nabla \barloss(\vparam_t) = \argmin_{\text{\vparam}} \,\, \vparam^\top \grad \barloss(\vparam_t) + \frac{1}{2} (\vparam-\vparam_t)^\top \vH(\vparam_t) (\vparam-\vparam_t) . \label{eq:grad_desc}
      %\quad\quad \text{ vs } \quad\quad
      %\vnatparam_{t+1} \leftarrow (1-\rho_t)\vnatparam_t + \rho_t \grad_{\vmeanparam_t} \myexpect_{q_{\vnatparam_t}} [-\barloss(\vparam)] .
\end{align}
\revision{Here, we denote by $\vH(\vparam_t)$ the Hessian of $\barloss$ at $\vparam_t$.}

The BLR has a similar interpretation where, at each step, a surrogate linear in $\vT(\vparam)$ is minimized at $q_t$ but the optimization takes place in the $\vmu$-space. 
This is obtained by using the `mirror descent' form of the BLR originally proposed by \citet{khan2017conjugate}. The following theorem states the result.
\vspace{.3cm}
\begin{thm}
   The BLR performs mirror descent over the VB objective $\mathcal{L}(q) = \myexpect_q[\barloss(\vparam)] - \entropy(q)$, that is, the steps can be written as follows:
\begin{equation}
   \begin{split}
      \vnatparam_{t+1} = (1-\rho_t) \vnatparam_t + \rho_t \tvlambda_t \quad &=\argmin_{\vmeanparam} \,\, \myang{\vmu, \nabla_{\vmeanparam_t} \elbofinal } + \frac{1}{\rho_t} \dkls{}{q}{q_t}. 
      %&= \arg\min_{\vmeanparam} \,\, \myang{\vmu, \vlambda_t - \tvlambda_t } + \frac{1}{\rho_t}  \mathbb{D}_{A^*}(\vmu\|\vmu_t).
   \end{split}
   \label{eq:blr_surrogate}
\end{equation}
\end{thm}
\begin{proof}
   \emph{
      The result can be proved straightforwardly by taking the gradient with respect to $\vmu$ and setting it to 0. To do so, we first write $\grad_{\vmeanparam} \elbofinal$ by using \cref{eq:grad_ent},
\begin{equation}
   \nabla_{\vmeanparam_t} \mathcal{L} = \nabla_{\vmeanparam_t} \myexpect_{q_t}[\barloss(\vparam)] - \nabla_{\vmeanparam_t} \entropy(q_t) = - \tvlambda_t + \vlambda_t .
   \label{eq:natgrad_elbo}
\end{equation}
\revision{We substitute this in the second expression in \cref{eq:blr_surrogate}, and then take its derivative,}
\begin{equation*}
   \nabla_{\vmeanparam} \sqr{ \myang{\vmu, \vlambda_t - \tvlambda_t } + \frac{1}{\rho_t} \dkls{}{q}{q_t} }  = 
   (\vlambda_t - \tvlambda_t) + \frac{1}{\rho_t}(\vlambda - \vlambda_t).
\end{equation*}
\revision{The above expression is obtained by noting that the derivative of the KL divergence is equal to the difference in the respective natural parameters.
   Setting this to 0, we get the first equality in \cref{eq:blr_surrogate}, which proves the equality between the two formulations.}}
\end{proof}
\revision{The mirror-descent formulation uses surrogates similar to the quadratic surrogates used in the Newton step (and also in gradient descent)} but the linearization is done with respect to $\vmu$, and the quadratic term is replaced by the KL divergence. To connect the BLR surrogate to the quadratic surrogate, we rewrite the optimization problem \revision{of \cref{eq:blr_surrogate}} in a form where the terms depending on $\barloss$ are isolated from the rest, and then specialize the surrogate to the case of a full-covariance Gaussian. We use \cref{eq:natgrad_elbo} to rewrite the BLR surrogate and then simplify in the second line by noting that $\myang{\vmu, \vlambda_t} = \myexpect_q[\log q_t(\vparam)]$,
\begin{equation}
   \begin{split}
      \vlambda_{t+1} &= \argmin_{\vmeanparam} \,\, \myang{\vmu, -\tvlambda_t + \vlambda_t } + \frac{1}{\rho_t} \dkls{}{q}{q_t} \\
      &= \argmin_{\vmeanparam} \,\, \myang{\vmu, -\tvlambda_t } + \myexpect_q[\log q_t(\vparam)] + \frac{1}{\rho_t} \dkls{}{q}{q_t} \\
      &= \argmin_{\vmeanparam} \,\, \myexpect_{q} \sqr{ \myang{\vT(\vparam), -\tvlambda_t } } + \frac{1}{\rho_t} \dkls{}{q}{q_t^{1-\rho_t}}. 
   \end{split}
\end{equation}
The last line is obtained by merging the last two terms in the second line into one.

Finally, we specialize the surrogate for Gaussians of the form $q_t = \gauss(\vm_t, \vS_t^{-1})$ by noting that the natural gradients can be expressed in terms of the gradient and Hessian. Specifically, we use the following result from \citet[Eq. 11]{khan2023bayesian}, 
\begin{equation}
   \tvlambda_t = - \myexpect_{q_t} \sqr{ \rnd{ \begin{array}{c} \nabla \barloss(\vparam) - \vH(\vparam) \vm_t \\ \half \vH(\vparam) \end{array} } }.
      \label{eq:natgrad_gauss}
\end{equation}
Noting that $\vT(\vparam) = (\vparam, \vparam\vparam^\top)$ for a full-covariance Gaussian, we get the following BLR surrogate which closely resembles the quadratic surrogate in \cref{eq:grad_desc},
\begin{equation}
    \begin{split}
       \myang{\vT(\vparam), -\tvlambda_t } &= \vparam^\top \myexpect_{q_t}[\nabla \barloss(\vparam) - \vH(\vparam) \vm_t] + \half \trace\sqr{ \vparam\vparam^\top \myexpect_{q_t}[\vH(\vparam)]} \\
       &= \vparam^\top \myexpect_{q_t}[\nabla \barloss(\vparam)] + \half (\vparam - \vm_t)^\top \myexpect_{q_t}[\vH(\vparam)] (\vparam-\vm_t) + \text{const.} 
    \end{split}
    \label{eq:blr_sgd}
\end{equation}
\revision{Compared to \cref{eq:grad_desc}, the} difference here is that the gradient and Hessian are now evaluated at and averaged over the samples from $q_t$, and the quadratic term is centered at $\vm_t$ instead of $\vparam_t$. The surrogate pays attention not only to the mean $\vm_t$ but also to its neighborhood. In this sense, this surrogate is more \emph{global} than the one used in the Newton step.

In fact, the quadratic surrogate in \cref{eq:grad_desc} can be obtained as a special case of the BLR surrogate by using the same approximations that are used by \citet{khan2023bayesian} to derive the Newton step. Essentially, we set $\rho_t = 1$ and then apply the delta method to approximate the expectations at the mean $\vm_t$ as shown below:
\begin{equation}
   \myexpect_{q_t}[\nabla \barloss(\vparam)] \approx \nabla \barloss(\vm_t)
   \quad\quad \text{ and } \quad\quad
   \myexpect_{q_t}[\vH(\vparam)] \approx \vH(\vm_t).
   \label{eq:delta}
\end{equation}
With $\rho_t = 1$, the KL term reduces to the entropy as shown in the second line below, then we substitute \cref{eq:blr_sgd}, apply the delta method, and simplify,  
\begin{align}
   \myexpect_{q} &\sqr{ \myang{\vT(\vparam), -\tvlambda_t } } + \frac{1}{\rho_t} \dkls{}{q}{q_t^{1-\rho_t}}  \label{eq:surrogate_conj}\\
   &= \myang{\myexpect_{q} [\vT(\vparam)], -\tvlambda_t }  - \entropy(q) \nonumber\\
       &= \myexpect_{q} \sqr{ \vparam^\top \myexpect_{q_t}[\nabla \barloss(\vparam)] + \half (\vparam - \vm_t)^\top \myexpect_{q_t}[\vH(\vparam)] (\vparam-\vm_t) } + \half \log|2\pi\vS| \nonumber\\ 
       &\approx \myexpect_{q} \sqr{ \vparam^\top \nabla \barloss(\vm_t) + \half (\vparam - \vm_t)^\top \vH(\vm_t) (\vparam-\vm_t) } + \half \log|2\pi\vS| \nonumber\\ 
   &= \vm^\top \nabla \barloss(\vm_t) + \half (\vm - \vm_t)^\top \vH(\vm_t) (\vm -\vm_t) + \half \trace\sqr{\vH(\vm_t) \vS^{-1} } + \half \log |2\pi\vS| \nonumber .
\end{align}
Optimizing the above approximate surrogate yields the quadratic surrogate in \cref{eq:grad_desc}. For example, collecting the terms that depend on $\vm$, we get the following:
\begin{equation}
   \vm_{t+1} = \argmin_{\text{\vm}} \,\, \vm^\top \nabla \barloss(\vm_t) + \half (\vm - \vm_t)^\top \vH(\vm_t) (\vm -\vm_t), 
\end{equation}
which is equivalent to \cref{eq:grad_desc}. \revision{Similarly,} optimizing with respect to $\vS$ is also straightforward and yields $\vS_{t+1} = \vH(\vm_t)$.  We see that, for a full-covariance Gaussian, the BLR surrogate recovers the quadratic surrogate obtained by Taylor's \revision{method. By using an arbitrary exponential-family form, the BLR surrogates give rise to novel surrogates, for example, those obtained by using Bernoulli or log-Normal forms.} 

\revision{For conjugate Bayesian models, the BLR surrogate reduces to the well known Fenchel conjugate. For instance, consider Bayes' rule given in \cref{eq:bayes_rule_add}.  By using the second line of \cref{eq:surrogate_conj}, we can equivalently rewrite \cref{eq:bayes_rule_add} as follows}
\begin{equation}
   \begin{split}
      \vlambda_{t+1} &= \argmin_{\vmeanparam} \,\, \myexpect_{q} \sqr{ \myang{\vT(\vparam), - ( \tvlambda_{\text{lik}} + \tvlambda_{\text{prior}} ) } } - \entropy(q) \\ 
      &= \argmin_{\vmeanparam} \,\, \myang{\vmu, - (\tvlambda_{\text{lik}} + \tvlambda_{\text{prior}} ) } + A^*(\vmu), 
   \end{split}
\end{equation}
where the second line follows by noting that the Shannon entropy $\entropy(q) = -A^*(\vmu)$ is also equal to the negative of the convex conjugate of the log-partition function. The result of \cref{eq:bayes_rule_add} is recovered by setting the derivative to 0,
\[
   - (\tvlambda_{\text{lik}} + \tvlambda_{\text{prior}} ) + \nabla A^*(\vmu) = 0
   \quad \implies \quad
   \vlambda_* = \tvlambda_{\text{lik}} + \tvlambda_{\text{prior}}.
\]
Fenchel conjugates are commonly used in machine learning \citep{blondel2020learning,paulus2020gradient} and their origins are attributed to convex optimization and information geometry, but \revision{Bayes is rarely mentioned.} The derivation \revision{highlights the close connections between these fields. Bayesian inference can be simply seen as the computation of} Fenchel conjugates whenever the joint distribution of the Bayesian model is log-linear with respect to $\vmu$.

We end this section with some historical remarks about the BLR surrogate. \revision{The forms in \cref{eq:blr_surrogate} and \cref{eq:bayes_rule_add} were first proposed by \citet{khan2017conjugate}. The connections to natural gradients were made explicit in a later work by \citet{khan2018fast1}, where the surrogates were referred to as the \emph{sites},} following the common nomenclature from the graphical model literature. 
The work of \citet{Opper:09} was the first to show the global nature of the VB solutions for Gaussian posteriors, which was later generalized to the exponential family by \citet{salimans2013fixed}, but neither used natural gradients to derive surrogates. %The fact that natural gradients play a special role in VB optimization is not yet sufficiently recognized by the community.

\section{A Large-Scale Implementation of VB for LLMs}
\label{sec:scaling}

Scaling VB to large problems has always been an issue. Directly optimizing the VB objective using deep-learning optimizers does not give satisfactory results, which has \revision{led to the belief that the} VB objective is inferior to those used in deep-learning training. Recent work by \citet{ShenDCNMBYGCKM24} addresses this issue by using the BLR. Natural gradients are powerful here because the natural-gradient update naturally takes a form similar to the popular Adam optimizer.
We will now briefly describe the result.

\citet{khan2023bayesian} showed that the BLR update takes a strikingly similar form to RMSprop and Adam. The RMSprop update is shown below
\begin{align} 
   \vparam_{t+1} &\leftarrow \vparam_t - \alpha \frac{ \widehat{ \grad}_{\text{\vparam}} \barloss(\vparam_t)}{\sqrt{\vv_{t+1}} + c \vone} , 
   \quad\text{where} \,\, \vv_{t+1} \leftarrow (1-\beta) \vv_t + \beta \sqr{\widehat{ \grad}_{\text{\vparam}} \barloss(\vparam_t)}^2. 
   \label{eq:rmsprop_update}
\end{align}
\revision{Similarly to stochastic gradient descent, the update uses the minibatch gradient (denoted by \smash{$\widehat{\grad}$}) with learning rate $\alpha$. The gradients are scaled by the square-root of a scale vector $\vv_t$. The constant $c$ is added to avoid dividing by 0. The scale vector contains a moving average of the squared gradient (element-wise square) with learning rate $\beta$. This is a heuristic used to reduce the stochastic noise due to minibatches
\citep{hintonTieleman}.}
%However, the BLR shows that the moving average naturally arises due to the use of entropy. The scaling with the scale vector then is naturally realized through a derivation similar to Newton's update.

A surprising result shown by \citet{khan2018fast} is that the BLR for the diagonal-covariance Gaussian family takes a form similar to RMSprop and, unlike RMSprop, no heuristics are required to derive it. The resemblance is primarily due to the information geometry of VB. Essentially, if we use $q_t = \gauss(\vm_t, \diag(\vs_t)^{-1})$ with a vector $\vs_t$, then the BLR reduces to a Newton-like form where the gradient and Hessian are evaluated at samples from $q_t$. Due to the diagonal
covariance, only the diagonal of the Hessian is needed.
The BLR update for this case takes the following form, 
\begin{align}
  \vm_{t+1} \leftarrow \vm_t -
   \rho_t \frac{\myexpect_{q_t}[\widehat{\grad}_{\text{\vparam}}
   \barloss(\vparam)]}{\vs_{t+1}} ,\text{ where } 
  \vs_{t+1} \leftarrow (1-\rho_t) \vs_t +
   \rho_t\,\, \myexpect_{q_t}[\diag(\widehat{\vH}(\vparam) ) ] .
  \label{eq:von_minibatch}
\end{align}
\revision{The update closely resembles RMSprop, but has some differences. Most notably, it evaluates gradients at samples from $q_t$ and replaces the squared gradient by the Hessian. Moreover, the update of the mean is a Newton-like step where square-root is not used for the scale vector. This update is referred to as the Variational Online Newton (VON) algorithm.}
\begin{figure}[t]
    \center
    %\subfigure[]{
       \includegraphics[width=4.5in]{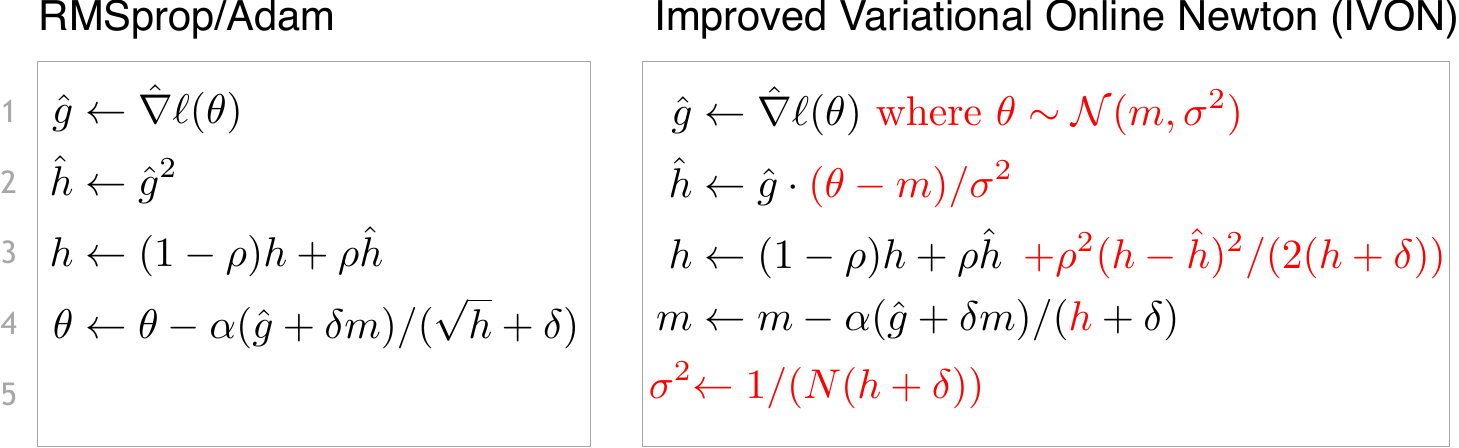}
        %\label{fig:bayes_robust_1}} \hspace{1cm}
    \caption{The figure shows a rough pseudo-code of RMSprop/Adam and compares it to the IVON algorithm which is the natural-gradient update of the BLR. We get a line by line correspondence between the two algorithms and the parts that are different are highlighted in red. The differences are essentially due to expectation in the VB objective which is approximated by one Monte-Carlo sample (in line 1). IVON computes the Hessian at this sample by using the
    reparameterization trick (line 2). The additional term in line 3 is due to a retraction obtained by using Riemannian gradient descent. The Hessian estimate is then used in a Newton-like update (line 4).}
    \label{fig:ivon}
\end{figure}

Recently, \cite{ShenDCNMBYGCKM24} used a Riemannian extension of VON, originally proposed by \citet{lin2020handling} and showed that it often performs comparably to Adam and sometimes can even beat it. They call the new algorithm the Improved VON algorithm (IVON as an abbreviation). Pseudo-code is shown in \cref{fig:ivon} where IVON is contrasted with the RMSprop/Adam pseudo-code. There is a line-by-line correspondence between the two methods, with some differences highlighted in red.

Because the IVON updates are almost identical to Adam, we can now efficiently compute natural gradients at large scale, for instance, for modern large language models (LLMs). The method overcomes a common criticism of natural gradients and Bayes, that they are too computationally intensive. The IVON algorithm defies such criticism. The posterior expectations are conveniently approximated \revision{with} one Monte-Carlo sample (line 1). \revision{The expectation of the Hessian is approximated by using} the reparameterization trick \citep{lin2019stein}, 
\[
   \myexpect_{q}[\vH(\vparam)] = 2 \nabla_{\text{\vS}^{-1}}\myexpect_{q}[\barloss(\vparam)] = \myexpect_{q}[\nabla \barloss(\vparam) \vS (\vparam - \vm)^\top],
\]
\revision{This example shows that not only can the natural gradients be computed efficiently at large scale, but also that natural gradient descent for VB can be implemented efficiently and in a form similar to deep-learning algorithms.} \cref{fig:teaser} shows the results for GPT-2 training (on OpenWebText) and ResNet-50 (on ImageNet), demonstrating that the natural-gradient update used in the BLR can in fact work well at large scale.

\begin{figure*}[t!]
   \centering
   \subfigure[GPT-2 on OpenWebText]{
       \includegraphics[width=.45\textwidth, keepaspectratio]{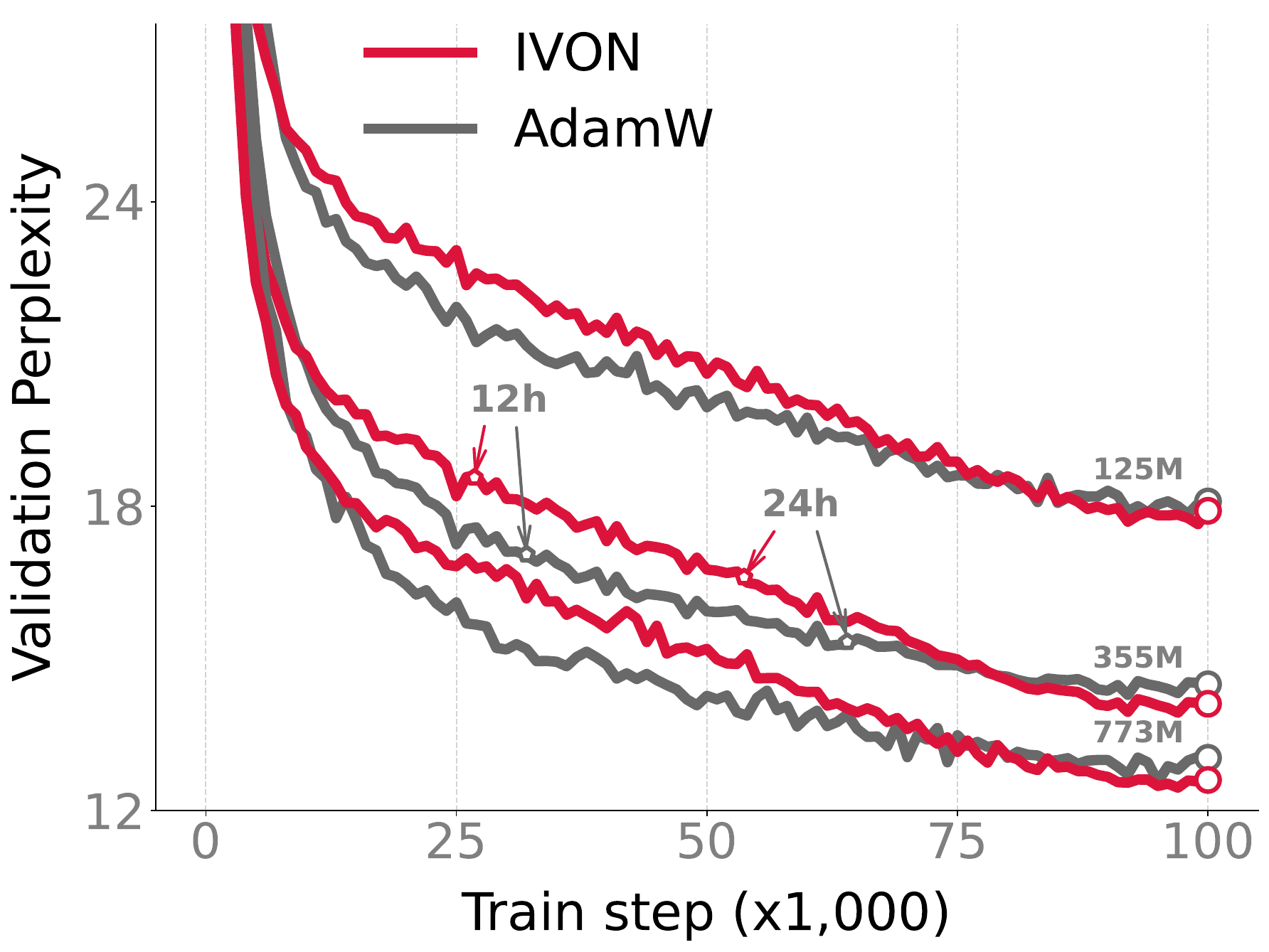}
     \label{fig:teaser_a}
   }
   \subfigure[ResNet-50 on ImageNet]{
     \includegraphics[width=.45\textwidth, keepaspectratio]{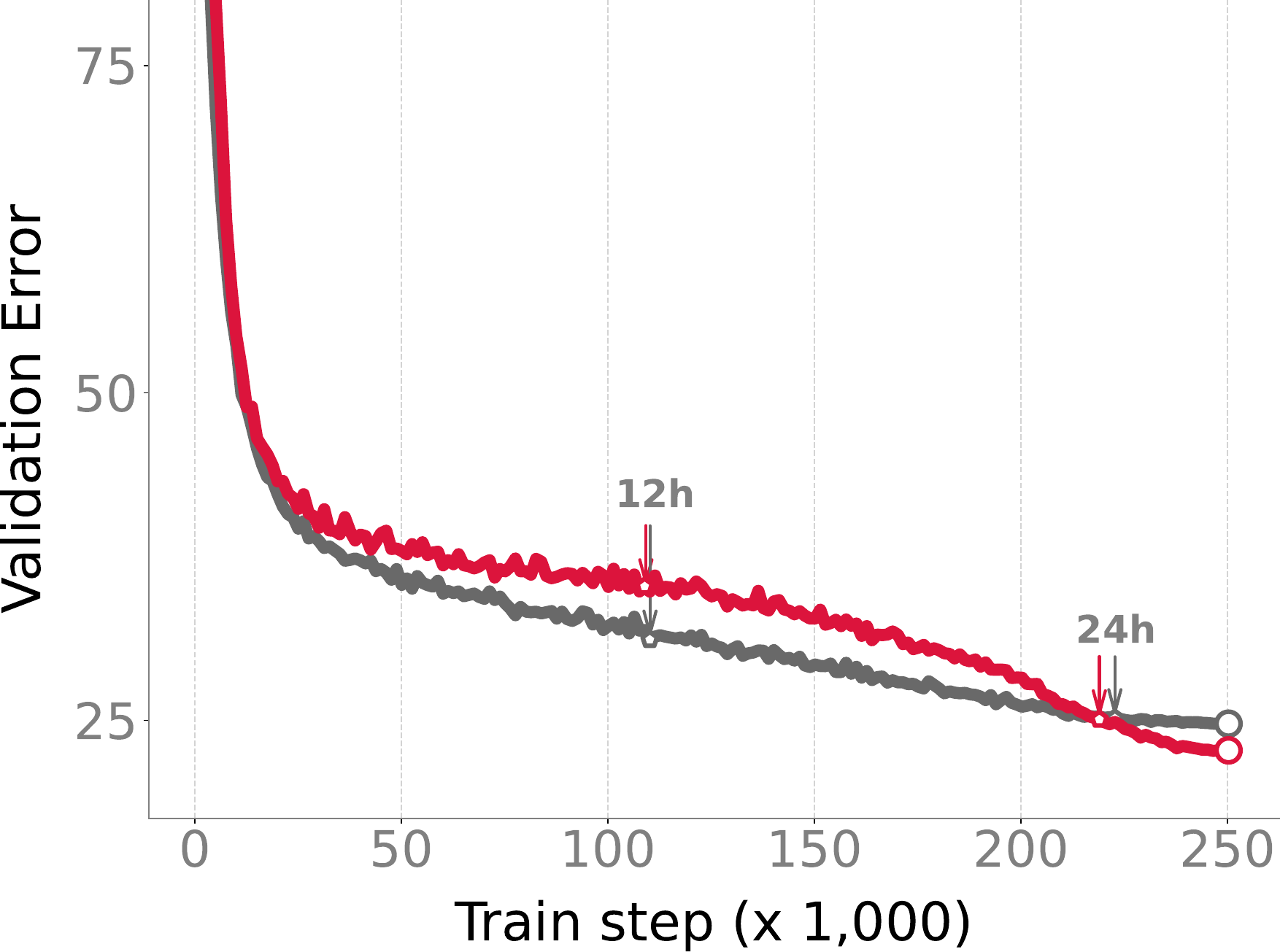}
     \label{fig:teaser_b}
   }
    \caption{Figures adopted from \citet{ShenDCNMBYGCKM24}. The first panel shows training of GPT-2 on OpenWebText and the second panel shows the same for ResNet-50 on ImageNet. Runtime in hours (h) is indicated by the arrows where we see that they are almost the same for IVON and AdamW. A small difference is due to the sampling in IVON. IVON obtains slightly better accuracy in the end.}
   \label{fig:teaser}
 \end{figure*}

%The natural gradient $\tvlambda_t$ can be seen as new information or \emph{innovation} being added to the current distribution $q_t(\vparam)$. 
%The update extends the existing notion of Bayesian filtering, such as those found in Kalman filtering \citep{Kalman1960ANA}, Viterbi algorithm \citep{stratonovich1965conditional}, and more generally dynamic programming \citep{bellman1966dynamic}.

\section{Conclusion}\label{sec:discussion}

In this paper, we highlight a fundamental connection between information geometry and VB and discuss its consequences for machine learning. For exponential-family posteriors, a VB solution always requires estimation or computation of natural gradients. We believe this result extends beyond exponential families, for example, \citet{linfast, lin2021tractable, kiral2023lie} show such results. We discuss several consequences of this fact, connecting
them to their common origin in the optimality condition of the VB objective. For instance, Bayes' rule, when seen as addition of natural gradients, achieves the optimality condition in one step. This is shown to be analogous to Newton's method. We further show that the natural-gradient update for VB generalizes the quadratic surrogates used in gradient-based optimization. Finally, we show how the natural-gradient updates resemble deep learning optimizers and how this fact can be exploited to
scale VB to large problems.

The information geometry of VB helps us to connect VB to other machine learning algorithms and makes it useful for improving aspects of deep learning. It helps us to overcome a common criticism that Bayesian and information geometric methods are computationally intensive and may not scale well to modern models. Thanks to the fundamental connection between information geometry and VB, elegant theory of Bayesian principles can now be applied to solve important practical problems.

\backmatter

%\bmhead{Supplementary information}
%
%If your article has accompanying supplementary file/s please state so here. 
%
%Authors reporting data from electrophoretic gels and blots should supply the full unprocessed scans for key as part of their Supplementary information. This may be requested by the editorial team/s if it is missing.
%
%Please refer to Journal-level guidance for any specific requirements.
%
\bmhead{Acknowledgments}

This work is supported by JST CREST Grant Number JP-MJCR2112. Many thanks to Sophia Sklaviadis (RIKEN, Tokyo, Japan, and Instituto de Telecomunica\c{c}\~{o}es and Instituto Superior T\'{e}cnico, Lisbon, Portugal) for proof reading and giving constructive feedback.

%\begin{appendices}

%\textbf{Conflict of Interest: } The author has co-authored several papers with Dr. Frank Nielsen. There are no other obvious conflict of interest to report at this moment.

\bibliography{refs}% common bib file

\end{document}